\documentclass[11pt,letter]{article}
\usepackage{amssymb}
\usepackage{amsthm,mathtools,bm,amsfonts,amssymb}
\usepackage{amsmath}
\usepackage{graphicx,algorithm,algorithmic,subfigure}
\usepackage[margin=1in]{geometry}
\newtheorem{definition}{Definition}
\newtheorem{theorem}{Theorem}
\newtheorem{lemma}{Lemma}
\title{Stable Memory Allocation in the Hippocampus:\\
	Fundamental Limits and Neural Realization}

\author{Wenlong Mou\thanks{Key Laboratory of Machine Perception, School of EECS, Peking University. Email: \texttt{mouwenlong@pku.edu.cn}} \\
	\and
	Zhi Wang\thanks{Key Laboratory of Machine Perception, School of Mathematics, Peking University. Email: \texttt{zhiwpku@gmail.com}}\\
	\and
	Liwei Wang\thanks{Key Laboratory of Machine Perception, School of EECS, Peking University. Email: \texttt{wanglw@cis.pku.edu.cn}}}
\begin{document}

\maketitle
\begin{abstract}
	
	It is believed that hippocampus functions as a memory allocator in brain, the mechanism of which remains unrevealed. In Valiant's neuroidal model \cite{hippo}, the hippocampus was described as a randomly connected graph, the computation on which maps input to a set of activated neuroids with stable size. Valiant proposed three requirements for the hippocampal circuit to become a stable memory allocator (SMA): stability, continuity and orthogonality. The functionality of SMA in hippocampus is essential in further computation within cortex, according to Valiant's model.
	
	In this paper, we put these requirements for memorization functions into rigorous mathematical formulation and introduce the concept of capacity, based on the probability of erroneous allocation. We prove fundamental limits for the capacity and error probability of SMA, in both data-independent and data-dependent settings. We also establish an example of stable memory allocator that can be implemented via neuroidal circuits. Both theoretical bounds and simulation results show that the neural SMA functions well.
\end{abstract}

	\section{Introduction}
	It is well known that the computational speed of human brain is several orders of magnitude slower than that of current computer, while almost everything human can do easily are still beyond the capability of the most powerful computers. Thus, the investigation into computational mechanisms of human brains leads a promising direction for theoretical aspect of computation, machine learning and neuroscience. In particular, the formation of memory, during which the neural circuit creates a representation for an object, was not fully understood, though it has been observed that the hippocampus plays a significant role in this process.\par
	A research line developed by Valiant \cite{neurobook,valiant2000neuroidal,valiant2006quantitative,valiant2014must} focused on the connectivity structure of neuroids, i.e., the computational abstraction of neurons, on which local computation can be performed. This formal and conservative model assumes the brain to be a randomly connected graph of neuroids. Each of them can be either active or inactive, depending on the states of neighbors that have a directed edge to it. Biological constraints are also taken into consideration, such as weak synaptic connections and sparsity of graphs. With biological evidences, this model captures the capability of realistic neurons. Specific supervised and unsupervised learning tasks such as inductive learning, rational expressions, and reasoning are viable in this model, using the vicinal algorithms \cite{memo,valiant2006quantitative,DBLP:journals/neco/FeldmanV09}. In Valiant's model, those complicated processes are based on basic operations of $JOIN$ and $LINK$. Recently, a new operation called $PJOIN$, namely, predictive $JOIN$, was proposed for unsupervised learning of patterns \cite{papadimitriou2015cortical,DBLP:conf/podc/PapadimitriouV15}. All those operations, as well as experimental evidences from neuroscience, requires the representation for an item to be allocated with a stable amount of neurons, which is done by hippocampus, before the computation in cortex.\par
	Valiant \cite{hippo} first discussed the role and computational mechanisms of hippocampus in this model. The paper stated that the function of hippocampus is to "identify the set of neurons in cortex at which a new compound concept or chunk will be represented, and to enable that set of neurons to take on that role". The stability of the amount of neurons allocated plays a key role in Valiant's neuroidal models in order to control each new chunk within a limited range so as to avoid the overall system becoming unstable. To guarantee this, three requirements for stable memory allocators (SMA) was proposed in their paper, namely, stability, continuity and orthogonality. Valiant pointed out the direction by discussing biological constraints and proposing an example in his paper, yet rigorous mathematical treatment is still in need. Our systematic analysis of stable memory allocators is therefore important in the neuroidal model.\par
	Memorization process creates neural representation for arbitrary set of input item, either randomly or through learning from samples. Performance and characteristics of such mapping functions have been vastly studied in machine learning, theoretical computer science and computational neuroscience. For example, in \cite{signcons}, a sign-consistent sparse JL transform is constructed to explain the formation of memory, under a different model of neurons. \cite{DBLP:conf/focs/ArriagaV99} proposed a neuron-friendly random projection for robust concept learning via dimension reduction. Hopfield networks \cite{hopfield} dynamically adjust network weights to create memory for new items. Locality sensitive hashing \cite{lsh} and binary embeddings \cite{DBLP:conf/icml/YiCP15} randomly maps input vector into a discrete space while preserving locality. Our work relates to them in terms of compression of information and preservation of locality.\par
	In this paper, we first manage to put Valiant's idea into a rigorous mathematical formulation. We give basic theoretical results and a simple construction in Section 2. By introducing concepts of error probability and capacity into our formulation, we are able to better characterize the fundamental possibilities and impossibilities of memory allocation. Their theoretical bounds are discussed under both data-dependent and data-independent settings; see Section 3.  While the network designed by Valiant works in special cases \cite{hippo}, it could not be generalized to the scale of real hippocampus. That's why we propose, in Section 4, a feed-forward neural network model to realize stable memory allocation in hippocampus, and at the same time consistent with the biological  constraints. In contrast to subtractive models \cite{hippo}, the inhibition patterns of neurons are divisive in our design. We show that this stable memory allocator can better achieves the requirements with reasonable capacity. 
	
	\section{Formulation of SMA}
	In this section, we will first refine the requirements for the stable memory allocation task \cite{hippo}, and put it into mathematical formulations.
	
	\subsection{Basic requirements and motivation}
	As proposed in \cite{hippo}, a valid stable memory allocator (SMA) for the cortex should subject to three basic requirements: stability, continuity and orthogonality. In the following we present high-level ideas in the formulation of SMA, and discuss basic requirements for a stable memory allocator to function well in Valiant's model.
	\begin{itemize}
		\item Stability: The mapping is stable in that for inputs within a wide range of activity levels, the output will have activity level differing little. This condition is fundamental in vicinal algorithms in the cortex, and also reasonable in neural representation of abstract concepts.
		\item Continuity: A noise in the input pattern should not cause huge variations in the output. Continuity can provide neural systems with robustness (error-correcting) and help to realize pattern recognition functions (nearest-neighbor search). 
		\item Orthogonality: With this property, distinct items will not be confused. Pairs of inputs that differ by much will be mapped to outputs that also differ sufficiently, so that they can be treated by cortex as distinct. We weaken the orthogonality requirement in \cite{hippo} and broaden the class of stable memory allocators, yet the biological functionality can still be guaranteed.
	\end{itemize}
	The randomness of mapping function plays an essential role in a stable memory allocator, since any data-independent deterministic allocation function will fail upon adversarial inputs: Suppose we have a fixed SMA which maps vectors in $\{0,1\}^n$ to a stable number of activated neurons. By pigeon-hole principle there must be $2^{\Omega(n)}$ vectors mapped to the same set of neurons. Among those colliding inputs, there must be a pair of vectors with Hamming distance $\Omega(n)$, which leads to the failure of orthogonality. On the contrary, a randomized SMA, like random projections, can fulfill these requirements for arbitrary input set of bounded size, with high probability. This observation is important and triggers the concept of error probability in our framework. We can thereby define the concept of capacity as the maximal number of objects that a stable memory allocator can handle with low error probability. In other words, the capacity of stable memory allocator stands for the number of items one can memorize without being confused.
	\subsection{Mathematical formulation}
	In our framework, SMA is formulated as a series of distributions with respect to size $n$ over function space. Although the growth of synapses is quite random, there are some underlying rules that govern the formation of random graph so that almost everyone can allocate memory for new incoming items without any difficulty, no matter the difference of inner connections for each individual. We therefore demand the three properties here to hold uniformly with high probability, which guarantees more robustness than the model based on expected value \cite{hippo}. It is also worth noticing that we do not make any assumptions on the structure or distribution of input vectors.
	\begin{definition}
		A series of distributions $\{\mathcal{H}_n\}$ over function space $\left\{0,1\right\}^n\rightarrow\left\{0,1\right\}^n$ is called a stable memory allocator (SMA) with parameter $\langle \delta,\mu,\kappa,A_n,B_n,r_n\rangle$ and capacity $K_n$, if $\forall \varepsilon>0$, $\exists N\in \mathbb{N}$ s.t. $\forall n>N$ $\forall  S_n\subset\left\{0,1\right\}^n \setminus \left\{\bm{x} | \vert \bm{x} \vert \leq \kappa n \right\},  \vert S_n\vert \leq K_n$, the following three properties hold
		\begin{equation}
		\begin{split}
		&	\mathop{Pr}_{h\sim\mathcal{H}_n}{\left\{\forall \bm{x}\in S_n, \vert h(\bm{x})\vert\ \in \left((1-\delta)r_n,(1+\delta)r_n\right)\right\}}>1-\varepsilon\\
		&\mathop{Pr}_{h\sim\mathcal{H}_n} \left\{\forall \bm{x}\in S_n, \bm{y} \in \{0,1\}^n, d_H(h(\bm{x}),h(\bm{y}))\leq \mu d_H(\bm{\bm{x},\bm{y}})\right\}>1-\varepsilon\\
		&\mathop{Pr}_{h\sim\mathcal{H}_n}{\left\{\forall \bm{x},\bm{y}\in S_n,d_H(\bm{x},\bm{y})>A_n,	 d_H(h(\bm{x}),h(\bm{y}))>B_n\right\}}>1-\varepsilon
		\end{split}
		\end{equation}
	\end{definition}
	It is natural to assume $\forall \bm{x},\bm{y}\in S_n,d_H(\bm{x},\bm{y})>A_n$ according to our interpretation of orthogonality. There are some requirements for parameters of SMA: as the output vector is stably sparse, the number of activated neurons in the output $r_n$ should satisfy $r_n/n \ll 1$. Also, $A_n$ and $B_n$ should increase with respect to size $n$ and one typical assumption is that they grow linearly: $A_n= an, B_n= br_n$.  Notice that $a$ is also a small number compared to 1, otherwise only few items are considered "truly different" and orthogonality would fail.\par
	Not all SMAs have good performance. When judging a SMA, we should first consider those parameters that define it: smaller Lipschitz factor $\mu$ in the continuity requirement means better preservation of input information and stronger error-correcting functions. SMA with larger $B_n$ and smaller $A_n$ can distinguish items better. The parameter $\kappa$ represents the lowest input density required. The following discussion, without pointing out explicitly, regards elements in $S_n$ as not too sparse (with density larger than $\kappa$) for all the $S_n$ mentioned. A small $\delta$ stands for good performance in stability. The extreme case is that $\delta=0$, which is  a strong version that realize stability precisely:
	\begin{description}
		\item[Strong stability]
		\begin{equation}\mathop{Pr}_{h\sim\mathcal{H}_n} \left\{\forall \bm{x}\in S_n, h(\bm{x})\in T_n^{(r_n)} \right\} >1-\varepsilon
		\end{equation}
		where $T_n^{(r_n)} \triangleq \left\{\bm{x}\in\left\{0,1\right\}^n\big| \vert\bm{x} \vert=r_n \right\}$ and we call a specific allocator function $h(\bm{x})$ strong stable with respect to input set $S_n$ if  $ \forall \bm{x} \in S_n, h(\bm{x})\in T_n^{(r_n)}$.
	\end{description}
	An SMA that satisfies the strong stability defined above is called a strong SMA.
	\subsection{Relating error probability to capacity}
	The key issue involved within capacity is error probability. The more items are there to be allocated, the more conflicts are likely to be happen in the outputs. Thus, one target of designing SMA is to maximize its capacity, i.e. to avoid the possible conflicts.\par
	Theorem~\ref{lower bound} reveals the relationship between pairwise error probability and capacity, in order to fulfill the requirements uniformly with high probability. In the design and analysis of SMA, this theorem serves as a useful tool to guarantee capacity.
	\begin{theorem}\label{lower bound}
		If a series of distributions $\{\mathcal{H}_n\}$ over function space $\left\{0,1\right\}^n\rightarrow \left\{0,1\right\}^n$ satisfies the following properties:
		\begin{itemize}
			\item For $\forall \bm{x}\in \left\{0,1\right\}^n$ fixed ,with probability at least $1-\varepsilon_n$:
			\begin{equation}\vert h(\bm{x})\vert\ \in \left((1-\delta)r_n,(1+\delta)r_n\right)
			\end{equation}
			\item For $\forall \bm{x}\in\left\{0,1\right\}^n$ fixed, with probability at least $1-\varepsilon_n$:
			\begin{equation}\forall \bm{y} \in \left\{0,1\right\}^n, d_H(h(\bm{x}),h(\bm{y}))\leq \mu d_H(\mathbf{\bm{x},\bm{y}})\end{equation}
			\item For $\forall \bm{x},\bm{y}\in \left\{0,1\right\}^n$ fixed, $d_H(\bm{x},\bm{y})>A_n$, we have: \begin{equation}
			\mathop{Pr}_{h\sim\mathcal{H}_n}{\left\{ d_H\left(h(\bm{x}),h(\bm{y})\right)>B_n\right\}}>1-\varepsilon_n
			\end{equation}
		\end{itemize}
		where $\lim\limits_{n\rightarrow 0} \varepsilon_n=0$. Then $\{\mathcal{H}_n\}$ is a strong SMA with parameters $\langle \mu,A_n,B_n,r_n \rangle$ and capacity at least $\frac{o(1)}{\sqrt{\varepsilon_n}}$, as $n$ goes to infinity.
	\end{theorem}
	The proof of this theorem estimates the overall error probability for a given set of items via union bound, the technical detail of which is deferred to the Appendix.
	\subsection{A simple example for SMA}
	The following simple construction serves as a proof for the existence of SMA. The basic idea is that we first randomly select $2r_n$ bits from the input item, then flip each of these bits with probability $\frac{1}{2}$. 
	
	Suppose that $\left\{i_1,i_2,\cdots,i_{2r_n}\right\}$ is uniformly distributed on all the $2r_n$-sized subsets of $\left\{1,2,\cdots,n\right\}$, and that $b_1,b_2,\cdots,b_{2r_n} \overset{i.i.d.}{\sim} \mathcal{B}\left(1,\frac{1}{2}\right)$. We define $h(\bm{x})$ along with its distribution $\mathcal{H}_n$ as follows, with operation $\oplus$ denoting the modulo 2 addition.
	\begin{equation}
	h(\bm{x})_k= \left\{
	\begin{array}{cc}
	b_j \oplus \bm{x}_k & k=i_j\\
	0 & otherwise 
	\end{array}
	\right.
	\end{equation}
	
	We argue that $\forall a,p,\delta\in (0,1)$, the construction above is an SMA with parameter $\langle \delta,0, 1 ,A_n=a\cdot n, (1-\delta) p A_n ,r_n=p\cdot n\rangle$ and capacity at least $o(1) e^{2pa^2\delta^2n}$, according to Theorem~\ref{lower bound}.
	
	Apparently, the flips would not change Hamming distance and naturally meet the requirement of continuity. As
	$\forall \bm{x} \in \{0,1\}^n$,$\vert h(\bm{x})\vert \sim \mathcal{B}\left(2r_n,\frac{1}{2}\right)$, we know stability by using Chernoff bound. For orthogonality, let $d=d_H(\bm{x},\bm{y})$. According to Chv{\'a}tal's estimation on tail probability  of hypergeometric distribution \cite{chvatal1979tail}, we have $\mathop{Pr}_{h\sim\mathcal{H}_n}\left\{ d_H(h(\bm{x}),h(\bm{y}))\leq (1-\delta)\frac{2r_n}{n}d \right\} \leq e^{-4r_n (\frac{d}{n})^2 \delta^2}$

	However, this example is still far from actual memorization mechanisms, since it could not be fit into a biologically plausible neural realization.
	
	\section{Theoretical analysis}
	In this section, we will consider the theoretical bounds for capacity and error probability of SMA, which are two important characteristics to judge its performance other than those parameters that define it. Lower bound of error probability and upper bound of capacity are first proposed. We also consider a different scheme where the memory allocator can be learned from input, and establish upper and lower bounds for the capacity. 
	\subsection{Lower bound of pairwise error probability}
	Lower bounds for pairwise error probability is important, since it indicates what capacity bound we can get via Theorem~\ref{lower bound}. In this subsection, we consider a slightly modified scheme, where the orthogonality condition is strengthened in Lipschitz form, and combined with continuity condition to constitute a bi-Lipschitz condition on the input set, as in Valiant's paper \cite{hippo}.
	\begin{theorem}\label{error probability lower bound}
		If $\{\mathcal{H}_n\}$ is a series of distribution over functions that satisfies the following requirements
		\begin{itemize}
			\item $\forall h \in$ supp$(\mathcal{H}_n)$, $\frac{|h(\bm{x})|}{r_n} \in (1-\delta,1+\delta)$
			\item $\forall \bm{x}, \bm{y} \in \{0,1\}^{n}$fixed, with probability at least $1-\varepsilon_n$:
			\begin{equation}\lambda d_H(\bm{x},\bm{y})\leq d_H(h(\bm{x}),h(\bm{y}))\leq \mu d_H(\bm{x},\bm{y})\end{equation}
		\end{itemize}
		then we have the lower bound for error 
		\begin{equation}\varepsilon_n = \Omega \left(\left[2 r_n \delta \binom{n}{r_n(1+\delta)}\right]^{-(\frac{\mu-\lambda}{\mu+\lambda})^2 \frac{1}{log n}} \right) 
\end{equation}
	\end{theorem}
	Basically, the theorem states information-theoretic limitations on a random mapping which can compress $\{0,1\}^n$ to a much smaller one while preserving distance. The proof is based on reduction to one-way public-coin randomized communication complexity, which is deferred to Appendix.
	\subsection{Upper bound of capacity}
	The following theorem gives an upper bound for the capacity when strong stability holds. This results mainly captures the limitations against achieving orthogonality posed by stability requirement.
	\begin{theorem}\label{upperbound for capacity}
		For an SMA with parameters $\langle 0,0,\mu,A_n,B_n,r_n\rangle$, if  $ max \big\{|S_n|\big|\forall \bm{x},\bm{y} \in S_n, d_H(\bm{x},\bm{y})>A_n\big\} \geq K_n$, then the capacity is upper bounded with:
		\begin{equation}\log K_n \leq \left(H(\alpha)-\alpha H(\beta)-(1-\alpha)H\left(\frac{\alpha\beta}{1-\alpha}\right)+o(1)\right)n 
		\end{equation}
		where $\alpha=\frac{r_n}{n}$, $\beta=\frac{B_n}{4r_n}$ and $H(\cdot)$ indicates entropy.
	\end{theorem}
	This theorem also discussed information-theoretic limitations on stable memory allocator, from a different perspective: the output space can only contain a limited number of "truly different" objects. The proof uses "sphere packing" arguments, which is deferred to the Appendix.
	\subsection{Theretical results for data-dependent allocators}
	Till now, our discussions are data-independent without exploring whether the memory allocator can adjust according to data. We focus in this part on the theoretical possibilities for a data-dependent memory allocator which can learn from items. Though it is impossible for hippocampus to be aware of all the items to encounter beforehand, our theoretical analysis in data-dependent setting reveals the possibilities and limits for online learning in memory allocation, where the memory allocator can adjust its mapping function after seeing each instance. In the following we propose both upper and lower bounds for the capacity in the data-dependent setting.\par
	The definition of SMA indicates that, for each fixed input set with limited size, there is at least one allocator $h$ that satisfies the three conditions regardless of the distribution $\mathcal{H}_n$ on the function space. Therefore, $\forall S_n, \exists$ $h$ dependent on $S_n$ such that $\forall \bm{x},\bm{y}\in S_n, d_H(\bm{x},\bm{y})>A_n \Rightarrow d_H(h(\bm{x}),h(\bm{y})\geq B_n$. An equivalent description of Theorem~\ref{upperbound for capacity} from a data-dependent perspective is as follows:\par
	If an allocator function is strong stable, then the number of "truly different" items it can discriminate would be upper bounded with:
	\begin{equation} \exp\left\{ \left(H(\alpha)-\alpha H(\beta)-(1-\alpha)H\left(\frac{\alpha\beta}{1-\alpha}\right)+o(1)\right)n \right\}\end{equation}
	Conversely, we consider that for a given set of "truly different" items $S_n$, whether an SMA that guarantees strong stability, orthogonality and uniform Lipschitz continuity exists. The following theorem guarantees the existence of data-dependent memory allocator with high capacity. We defer the technical proof to Appendix.\\
	\begin{theorem} \label{existence of hashing function}
		Given input set $S_n \subset \left\{0,1\right\}^n$ that satisfies $\forall \bm{x},\bm{y}\in S_n, d_H(\bm{x},\bm{y})> A_n$. If $\log{\vert S_n \vert }\leq \frac{1}{2} \left(H(\alpha)-\alpha H(\beta)-(1-\alpha)H \left(\frac{\alpha\beta}{1-\alpha}\right)\right)n$, then there exists a function $h_n: \left\{0,1\right\}^n\rightarrow \{0,1\}^n $ such that
		\begin{equation}
		\begin{split} 
		\forall \bm{x} \in S_n, h(\bm{x}) \in T_n^{(r_n)}; \forall \bm{x},\bm{y}\in S_n, d_H(h_n(\bm{x}),h_n(\bm{y}))> B_n;\\
		\forall \bm{x} \in S_n, \forall \bm{y} \in  \{0,1\}^n, d_H(h_n(\bm{x}),h_n(\bm{y})) \leq \frac{8 r_n}{A_n} d_H(\bm{x},\bm{y})
		\end{split}
		\end{equation}
		where $r_n=|S_n|$, $\alpha=\frac{r_n}{n}$, $\beta =\frac{B_n}{2r_n}$ and $H(\cdot)$ is the entropy.
	\end{theorem}
	Notice that the number of input items in Theorem~\ref{existence of hashing function} is approximately the square root of the upper bound indicated by Theorem~\ref{upperbound for capacity}. The lower bound is strong in this sense.
	\section{A feedforward neural network for SMA}
	In this section, we propose a feedforward neural network model that realizes SMA functionality. Compared with Valiant's previous neural circuits \cite{hippo}, not only does our construction achieve better performance of SMA, with capacity guarantees, but also it is biologically more reasonable, in terms of number of synapse connected to a neuron and their strengths.\par
	\subsection{Biologically reasonable assumptions}
	We adopt the randomly-connected feed-forward network architecture and unmodifiable synapse in Valiant's paper \cite{hippo}. Actually, as Valiant has suggested in his paper \cite{hippo}, there are also experimental evidences supporting this architecture. The information flow within hippocampus appears to be uni-directional, with much less reciprocal connections \cite{andersen2006hippocampus}; it has also been reported that both modifiable and unmodifiable synapse exists in the hippocampus \cite{debanne1999heterogeneity}. As in Valiant's paper, our theory focuses on explaining the mechanism of memory allocation component, with conservative assumptions on neurons, though the overall infrastructure of hippocampus can be much more complicated.\par
	A major difference we made in the design of neural circuit is that, the inhibition patterns of neurons are divisive, instead of subtractive. Whether the way of inhibition is subtractive or divisive is an important open question in neuroscience \cite{doiron2001subtractive,mejias2014subtractive}. In divisive inhibition model, each neuron is activated if and only if the ratio of its total positive synaptic inputs $\sum_{+} W_i$ to the sum of all the synaptic weights that come from input neurons $(\sum_{+} W_i) +(\sum_{-} W_i)$ is larger than the threshold value $C$, i.e. $\frac{\sum_{+} W_i}{\sum_{+} W_i +\sum_{-} W_i}>C$. For the purpose of memory allocation, subtractive models \cite{hippo} have to narrow down the range of density in each layer and use contraction mapping theorem for stability. Divisive inhibition, on the contrary, enables us to control the expected level of activity in second layer directly, for a wide range of input density.\par
	\subsection{Description of the network}
	Our neural realization is based on a unidirectional three-layer network architecture. For simplicity we assume each layer has $n$ neurons. Firing neurons in the first layer represents the input vector $\bm{x}\in\{0,1\}^n$, and the output $h(\bm{x})$ is taken from the third layer. For $k\in\{1,2\}$, the edges from $k$-th layer to $(k+1)$-th layer is randomly and independently connected: each edge exists with probability $2p$, and the chances for it to become positive and inhibitive are both $p$. The threshold value in divisive inhibition is denoted as $C_1$ and $C_2$, for second and third layer neurons respectively.
	\subsection{Proof of SMA functionality}
	In this subsection, we give proofs on the validity of our construction as a stable memory allocator, with reasonable parameters and capacity.\par
	We first give the following technical lemmas needed in our analysis. Lemma~\ref{middle layer stablity}, ~\ref{middle one-bit change probablity}, ~\ref{CLT second-layer change probability} and ~\ref{CLT third-layer change probability} described the consequences of passing through one layer in our network model. We defer the proofs for the lemmas and theorem to Appendix.
	\begin{lemma}\label{middle layer stablity}
		For $\forall \bm{x}\in\left\{0,1\right\}^n$ fixed, let $\beta\in \mathbb{R}^n$ with all entries $i.i.d.$, $Pr\{\beta_i=1\}=Pr\{\beta_i=-1\}=p$, $Pr\{\beta_i=0\}=1-2p$. Then we have:
		\begin{equation}0\leq \frac{1}{2}-\mathop{Pr}\{\beta^T\bm{x}>0\}\leq \frac{1}{2\sqrt{\pi p |\bm{x}|}}+\frac{1}{2}e^{-2(2\ln 2-1)p|\bm{x}|}
		\end{equation}
	\end{lemma}
	The following lemma estimates how one-bit difference in the input influence the output.
	\begin{lemma}\label{middle one-bit change probablity}
		For $\forall \bm{x},\bm{y}\in\left\{0,1\right\}^n$ fixed, with $d_H(\bm{x},\bm{y})=1$, let $\beta\in \mathbb{R}^n$ with all entries $i.i.d.$, $Pr\{\beta_i=1\}=Pr\{\beta_i=-1\}=p$, $Pr\{\beta_i=0\}=1-2p$. For $\forall \bm{v}\in\{0,1\}^n$, let  $g(\bm{v})=I[\beta^T\bm{v}]>0$. Then we have:
		\begin{equation}Pr\left\{g(\bm{x})\neq g(\bm{y})\right\}\leq 2\sqrt{\frac{p}{\pi |\bm{x}|}}+2pe^{-2(2\ln 2-1)p|\bm{x}|}
		\end{equation}
	\end{lemma}
	For a more general case, we cannot write the probability of changed bits in closed form, while its asymptotic behavior can be estimated using central limit theorem, as stated in the following two lemmas.
	\begin{lemma}\label{CLT second-layer change probability}
		For $\forall \bm{x},\bm{y}\in\left\{0,1\right\}^n$ fixed, with $d_H(\bm{x},\bm{y})=L$, let $\beta\in \mathbb{R}^n$ with all entries $i.i.d.$, $Pr\{\beta_i=1\}=Pr\{\beta_i=-1\}=p$, $Pr\{\beta_i=0\}=1-2p$. For $\forall \bm{v}\in\{0,1\}^n$, let  $g(\bm{v})=I[\beta^T\bm{v}]>0$. Then we have:
		\begin{equation}Pr\left\{g(\bm{x})\neq g(\bm{y})\right\}\approx\frac{1}{\pi}\cos^{-1}\frac{|\bm{x}\cap \bm{y}|}{\sqrt{|\bm{x}|\cdot |\bm{y}|}}
		\end{equation}
	\end{lemma}
	\begin{lemma}\label{CLT third-layer change probability}
		For $n$ independent and identically distributed pairs $\{(x_i,y_i)\}_{i=1}^{n}$ with $Pr\{x_i=1\}=Pr\{y_i=1\}\approx \frac{1}{2}$ and $Pr\{x_i\neq y_i\}=\eta$, we concatenate them and get two random vectors $\bm{x},\bm{y}\in \{0,1\}^n$. Let $\beta\in \mathbb{R}^n$ with all entries $i.i.d.$, $Pr\{\beta_i=1-c\}=Pr\{\beta_i=-c\}=p$, $Pr\{\beta_i=0\}=1-2p$. For $\forall \bm{v}\in\{0,1\}^n$, let  $g(\bm{v})=I[\beta^T\bm{v}]>0$. We have $Pr\{g(\bm{x})\neq g(\bm{y})\}$ is approximately equal to:
		\begin{equation}
		\frac{2}{\sqrt{1-\frac{1}{2}\eta}}\int_{\mathbb{R}}{\varphi(\frac{t}{\sqrt{1-\frac{1}{2}\eta}})\Phi(\frac{-1}{\sqrt{\eta}}\big|\sqrt{2}t-\frac{(2c-1)p\sqrt{n}}{\sqrt{v(c,p)}}\big|)}dt
		\end{equation}
		($\varphi(\cdot)$ and $\Phi(\cdot)$ denotes the density and cumulative density function for $\mathcal{N}(0,1)$, respectively)\\
		($v(c,p)=(2c^2-2c+1)p-(1-2c)^2p^2$ is variance of $\beta_i$)
	\end{lemma}
	By combining them together, we conclude the functionality of neural stable memory allocator:
	\begin{theorem}\label{neural SMA}
		Given $r_n,A_n=a\cdot n,s_0$, $\forall \gamma\in(0,1)$, we construct the neural network with parameters: $p=\frac{1}{n^{\gamma}}$, and $C_1=\frac{1}{2}$, $C_2=\frac{1}{2}-
		\frac{s}{2\sqrt{2-s^2}}$, where $s=\Phi^{-1}\left(\frac{r_n}{n}\right)\sqrt{\frac{2}{np}}$. Let $\mathcal{H}$ denote the mapping function defined by this neural network. There exists a constant $b$ depending solely on $a$, and $\mu_n=Z_0r_ns_0^{-\frac{1}{4}}n^{-\frac{1+\gamma}{4}}(\log n)^2$ with some absolute constant $Z_0$, such that $\forall \bm{x},\bm{y},\bm{z}\in \{0,1\}^n$with $|\bm{x}|,|\bm{y}|,|\bm{z}|\geq s_0n$, and $d_H(\bm{x},\bm{z})>A_n)$, the following hold if $n$ is large enough:
		\begin{equation}
		\begin{split}
		&\mathop{Pr}_{h\sim \mathcal{H}}\left\{\frac{|h(\bm{x})-r_n|}{n}>\epsilon\right\}\leq \exp\left\{-2n\left(\epsilon-\frac{\log(n/r_n)}{\sqrt{s_0n^{1-\gamma}}}\right)^2\right\}\\
		&\mathop{Pr}_{h\sim \mathcal{H}}\left\{\frac{d_H(h(\bm{x}),h(\bm{y}))}{\sqrt{d_H(\bm{x},\bm{y})}}\geq t\mu_n \right\}\leq \exp\left\{-(t\ln t-(t-1))\mu_n\sqrt{d_H(\bm{x},\bm{y})}\right\}.\quad \forall t\geq 1\\
		&\mathop{Pr}_{h\sim\mathcal{H}}\left\{d_H(h(\bm{x}),h(\bm{z}))\leq  (b-\epsilon)r_n\right\}	\leq e^{-2n\epsilon^2}
		\end{split}
		\end{equation}
	\end{theorem}
	Assume $r_n$ grows linearly with $n\rightarrow +\infty$, while $\epsilon, b$ remains absolutely constant, and $t=\Theta(n^{\delta})$ with small absolute constant $\delta$, the error probabilities in Theorem~\ref{neural SMA} diminish quasi-exponentially as $n$ increases. Combined with Theorem~\ref{lower bound}, we conclude that the neural stable memory allocator achieves a capacity which is exponential in $poly(n)$. The continuity requirement in this neural SMA deviates from standard definition: on the one hand, though the Lipschitz factor $t\mu_n$ grows unboundedly as $n\rightarrow+\infty$, it is much smaller than $r_n$ so that error-correcting functions can still be performed; on the other hand, the theorem provides more guarantees than Lipschitz constant, since the number of different output bits grows linearly with $\sqrt{d_H(\bm{x},\bm{y})}$ instead of $d_H(\bm{x},\bm{y})$, making it more robust for input errors of more than constant bits. 
	\subsection{Simulation results}
	In this subsection, we verify the theoretical guarantees via simulating the neural network. Numerical results show that, the neural stable memory allocator is consistent with theoretical predictions, and thus functions well as an SMA.\par
	We construct three-layer network as described in this section. Guided by Theorem~\ref{neural SMA}, parameters are set as follows:
	$$n=10^5,p=2.5\times 10^{-3},C_1=0.5,C_2=0.57$$
	To illustrate the stability of allocated neurons, we observe the proportion of activated neurons in second and third layer for input density ranging from 0.05 to 0.5, as plotted in Figure 1.a. The density of firing neurons is restricted in $(0.46,0.49)$ in the second layer, and further narrowed down in range $(0.014,0.017)$ in the third layer. By adjusting $C_2$ we can also control the density in output layer arbitrarily.
	\begin{figure}[H]
		\small
		\subfigure[Stability]{
			\centering
			\includegraphics[width=0.48\linewidth]{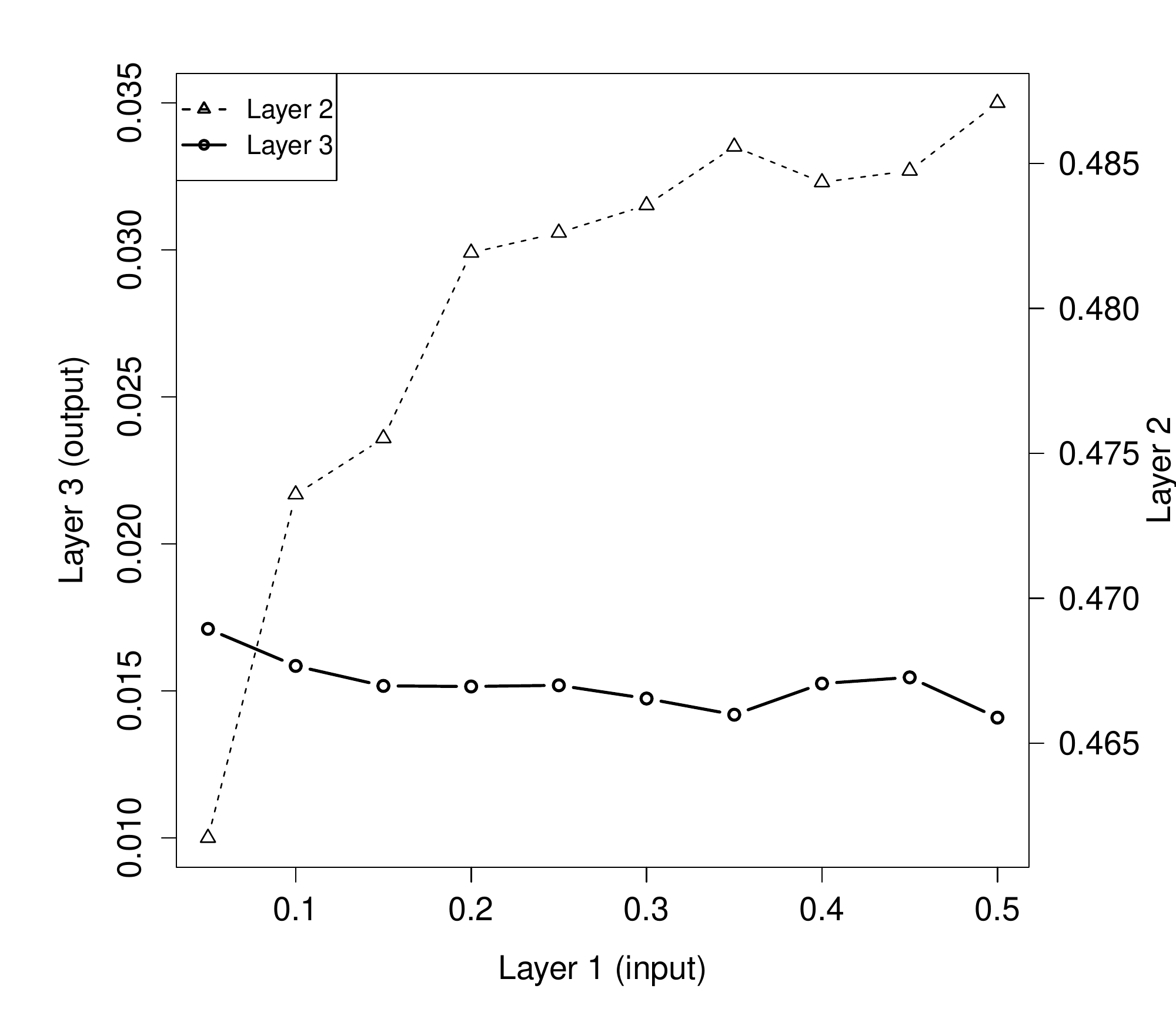}
		}
		\subfigure[Expansion rates]{
			\centering
			\includegraphics[width=0.47\linewidth]{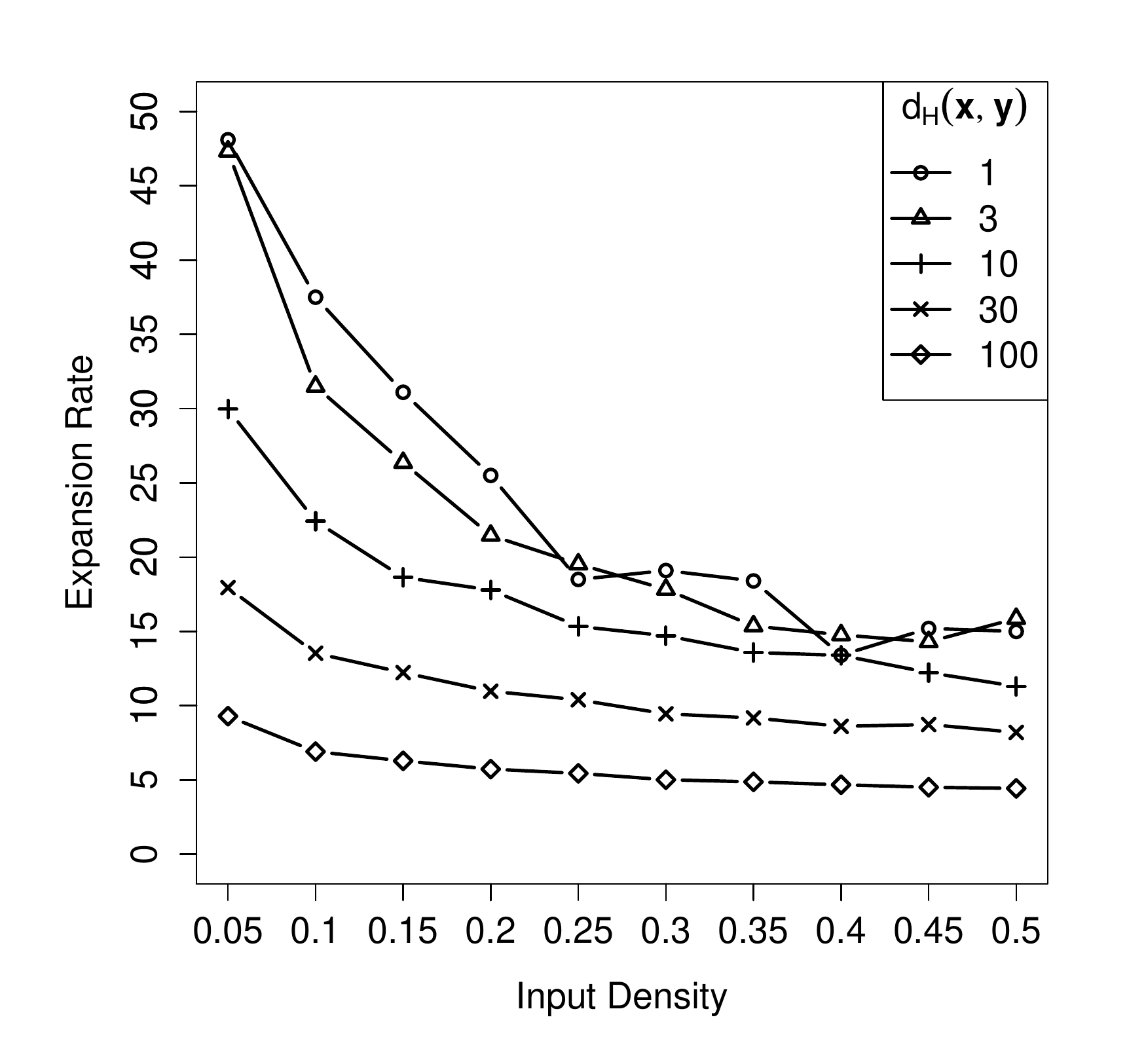}
		}
		\caption{Stability and expansion rate of the neural SMA}
	\end{figure}
	
	Figure 1.b shows the relationship between expansion rate $\frac{d_H(h(\bm{x}),h(\bm{y}))}{d_H(\bm{x},\bm{y})}$ and the proportion of activated input neurons, which demonstrates continuity and orthogonality properties of our memory allocator. The values shown in the figure is by taking average from $10$ randomly chosen inputs. Different curves stand for the behaviors of expansion rate for $d_H(\bm{x},\bm{y})=1,3,10,30,100$, respectively. As the figure shows, difference in inputs for a few bits will not cause the large Hamming distance between outputs, which verifies continuity. On the other hand, the expansion rate is lower bounded with some positive constant, which illustrates orthogonality. It is also worth noticing that expansion rate decreases as $d_H(\bm{x},\bm{y})$ increases, which is consistence with continuity result in Theorem~\ref{neural SMA}. Therefore, our experimental results verify the validity of Theorem~\ref{neural SMA}, and demonstrate the performance of our construction as a stable memory allocator.
	\section{Conclusion and open questions}
	In this paper we give definitions, fundamental limits and a neural realization for stable memory allocation in the hippocampus, under Valiant's neuroidal framework. We first push ahead rigorous formulations from the three properties of SMA that Valiant has proposed in his paper \cite{hippo}. Some important characteristics of SMAs in our definition such as capacity and error probability are discussed with their upper and lower bounds presented. We also explore the theoretical possibilities of memory allocators which can learn from input items.\par
	We explicitly construct a stable memory allocator based on a randomly-connected feed-forward neural network with biologically sound parameters. We use divisive inhibition of neurons in our model, which allows us to obtain arbitrary level of stability and preserves continuity better. Rigorous proofs are given for the performance of this SMA. We also verify our theory through computer simulations.\par
	There are a lot of promising future works we can do. By making full use of the three conditions in definition, we can possibly give tighter bounds for theoretical limits of SMAs. On the other hand, though we have constructed a neural network model that realizes SMA, it is far from optimal. The question whether we can construct a neural stable memory allocator with constant Lipschitz factor for continuity remains open. Besides, it is worthwhile to investigate more complicated computational functions of the hippocampus, such as the interaction between hippocampus and cortex, and the retrieval of allocated memory.

\section*{Appendix}

\subsection*{Proof of Theorem~\ref{lower bound}}
\begin{proof}
	We bound the overall error probability using pairwise error probabilities via union bound. For orthogonality, we have:
	\begin{align*}
	&\mathop{Pr}_{h\sim\mathcal{H}_n}{\left\{\forall \bm{x},\bm{y}\in S_n,d_H(\bm{x},\bm{y})>A_n, d_H(h(\bm{x}),h(\bm{y}))>B_n\right\}}\\
	=&1-\mathop{Pr}_{h\sim\mathcal{H}_n}{\left\{\exists \bm{x},\bm{y}\in S_n,d_H(\bm{x},\bm{y})>A_n, d_H(h(\bm{x}),h(\bm{y}))<B_n\right\}}\\
	=&1-\mathop{Pr}_{h\sim\mathcal{H}_n}{\left\{\bigcup_{\bm{x},\bm{y}\in S_n,d_H(\bm{x},\bm{y})>A_n} \{d_H(h(\bm{x}),h(\bm{y}))<B_n\}\right\}}\\
	\geq &1-\sum_{\bm{x},\bm{y}\in S_n, d_H(\bm{x},\bm{y})>A_n}\mathop{Pr}_{h\sim\mathcal{H}_n}{\left\{d_H(h(\bm{x}),h(\bm{y}))<B_n\right\}}\\
	\geq &1-\vert S_n\vert^2 \varepsilon_n
	\end{align*}
	Similarly, for continuity and stability we have:
	\begin{equation}\mathop{Pr}_{h\sim \mathcal{H}_n}\{ \forall \bm{x}\in S_n,\bm{y}\in \{0,1\}^n, d_H(\bm{x},\bm{y})\leq \mu d_H(\bm{x},\bm{y}) \}\geq 1-|S_n| \varepsilon_n\end{equation}
	\begin{equation}\mathop{Pr}_{h\sim \mathcal{H}_n}\{ \forall \bm{x}\in S_n, |h(\bm{x})|\in (r_n(1-\delta),r_n(1+\delta)) \}\geq 1-|S_n| \varepsilon_n\end{equation}
	For any sequence $\xi_n\rightarrow 0$, let $|S_n|=K_n=\xi_n\frac{1}{\sqrt{\varepsilon_n}}$. The probability of violating any of the requirements is $O(\xi_n)$, which goes to zero. Therefore, the capacity of this SMA is at least $\frac{o(1)}{\sqrt{\epsilon_n}}$.
\end{proof}

\subsection*{Proof of Theorem~\ref{error probability lower bound}}
\begin{proof}
	The proof is based on the lower bound of one-way randomized approximate communication complexity of Hamming distance by \cite{jlopt}. They showed that:
	If Alice has  vector $\bm{x}$ of length $n$, and Bob has vector $\bm{y}$ of length $n$, to estimate $d_H(\bm{x},\bm{y})$ within a factor of $1\pm \delta$ with probability at least $1-\varepsilon$, Alice has to send at least $\Omega \left( \frac{1}{\delta^2}\log\frac{1}{\varepsilon} log n\right)$ bits to Bob.
	
	We can use SMA to construct an one-way communication protocol: For a given SMA described in the theorem, a common hashing function $h\sim \mathcal{H}_n$ for Alice and Bob is sampled first. Upon receiving the vector $\bm{x}$, Alice computes $h(\bm{x})$ and sends it to Bob. Bob has vector $\bm{y}$ and gets $h(\bm{x})$, so he can compute  $d_H(h(\bm{x}),h(\bm{y}))$. 
	Because there are $\sum_{j=r_n(1-\delta)}^{r_n(1+\delta)} \binom{n}{j} \leq 2r_n\delta \binom{n}{r_n(1+\delta)}$ different possible values for $h(\bm{x})$, the number of bits needed to encode these values is at least:
	\begin{equation}\label{encode bit}
	K=\log\left(2r_n \delta \binom{n}{r_n(1+\delta)}\right)
	\end{equation} We can write the Lipschitz condition in normalized form:
	\begin{equation}\mathop{Pr}_{h\sim\mathcal{H}_n} \left((1-\frac{\mu-\lambda}{\mu+\lambda})  d_H(\bm{x},\bm{y}) \leq \frac{2}{\mu+\lambda}d_H(h(\bm{x}),h(\bm{y})) \leq (1+\frac{\mu-\lambda}{\mu+\lambda}) d_H(\bm{x},\bm{y}) \right) > 1-\varepsilon_n 
	\end{equation}
	By plugging into the result of \cite{jlopt}, we have 
	\begin{equation}\label{one way communication complexity}
	K=\Omega \left( (\frac{\mu+\lambda}{\mu-\lambda})^2 log \frac{1}{\varepsilon_n} log n \right)
	\end{equation}
	By combining equation~(\ref{encode bit})(\ref{one way communication complexity}), we have 
	\begin{equation}\varepsilon_n = \Omega \left(\big[2 r_n \delta \binom{n}{r_n(1+\delta)}\big]^{-(\frac{\mu-\lambda}{\mu+\lambda})^2 \frac{1}{log n}} \right) 
	\end{equation} 
\end{proof}

\subsection*{Proof of Theorem~\ref{upperbound for capacity} }
\begin{proof}
	For a strong SMA with $A_n$ not too large, $ max \big\{|S_n|\big|\forall \bm{x},\bm{y} \in S_n, d_H(\bm{x},\bm{y})>A_n\big\} \geq K_n$, we can select set of input items $S_n \subseteq \{0,1\}^n$ that satisfies $|S_n|=K_n$ and $\forall \bm{x},\bm{y} \in S_n, d_H(\bm{x},\bm{y})>A_n$. According to the definition of strong SMA, there exists a hashing function $h$ such that $\forall \bm{x} \in S_n, h(\bm{x}) \in T_n^{(r_n)} $ and $\forall \bm{x}, \bm{y} \in S_n, d_H(\bm{x},\bm{y})>B_n$.
	
	Let $U(\bm{x},r)$ denote $\left\{\bm{x'}\in T_n^{(r_n)}\vert d_H(\bm{x},\bm{x}')\leq r\right\}$, we have
	\begin{equation}\forall \bm{x},\bm{y} \in S_n, U\left(h(\bm{x}),\frac{B}{2}\right)\cap U\left(h(\bm{y}),\frac{B}{2}\right)=\varnothing
	\end{equation}
	Thus, we have
	\begin{equation}
	\sum_{\bm{x}\in S_n}{\big| U\left(h(\bm{x}),\frac{B_n}{2}\right) \big|}\leq\vert T_n^{(r_n)}\vert= \binom{n}{r_n}
	\end{equation}
	To estimate the number of elements in set $\big| U\left(h(\bm{x}),\frac{B_n}{2}\right) \big|$, notice that to guarantee $h(\bm{x}) \in T_n^{(r_n)}$, we can change $k \leq B_n/4$ bits at most from 0 to 1 and from 1 to 0 at the same time. 
	Therefore,\begin{equation} \left| U\left(h(\bm{x}),\frac{B_n}{2}\right) \right| = \sum_{k=1}^{\frac{B_n}{4}} \binom{r_n}{k}\cdot \binom{n-r_n}{k}\end{equation}
	Also, \begin{equation}B_n \leq d_H(h(\bm{x}),h(\bm{y})) \leq 2r_n \Longrightarrow k \leq \frac{B_n}{4} \leq \frac{r_n}{2}
	\end{equation}
	The sparsity of output suggests that $r_n \ll n$, hence
	\begin{equation} \label{monotonous}
	\frac{\binom{r_n}{k+1}\cdot \binom{n-r_n}{k+1}}{\binom{r_n}{k}\cdot \binom{n-r_n}{k}}=\frac{1}{(k+1)^2} (r_n-k)(n-r_n-k) \geq \frac{1}{(\frac{r_n}{2}+1)^2} \cdot \frac{r_n}{2}(n-r_n-\frac{r_n}{2}) \gg 1
	\end{equation}
	and \begin{equation}\binom{r_n}{\frac{B_n}{4}}\cdot \binom{n-r_n}{\frac{B_n}{4}} \leq \big| U\left(h(\bm{x}),\frac{B_n}{2}\right) \big|\end{equation} is a good estimation for $\big| U\left(h(\bm{x}),\frac{B_n}{2}\right) \big|$ ,and we have
	\begin{equation} \label{bound for S_n}
	|S_n| \leq \frac{\binom{n}{r_n}}{\binom{r_n}{\frac{B_n}{4}}\cdot \binom{n-r_n}{\frac{B_n}{4}}}
	\end{equation}
	As $n,r_n$ and $B_n$ are large, we can use Stirling's formula $\log n!=n\log n-n+O(\log n)$ to approximate $\binom{n}{r_n}$:
	\begin{equation}log \binom{n}{r_n}= (-\frac{r_n}{n}log\frac{r_n}{n}-(1-\frac{r_n}{n})log(1-\frac{r_n}{n})+o(1))n= (H(\alpha)+o(1))n\end{equation}Similar results can be derived for $\binom{r_n}{\frac{B_n}{4}}$ and $\binom{n-r_n}{\frac{B_n}{4}}$. By using inequality~(\ref{bound for S_n}), we get
	\begin{equation}\log K_n \leq \left(H(\alpha)-\alpha H(\beta)-(1-\alpha)H\left(\frac{\alpha\beta}{1-\alpha}\right)+o(1)\right)n \end{equation}
	where $\alpha=\frac{r_n}{n}$, $\beta=\frac{B_n}{4r_n}$ and $H(p)=-p\log p-(1-p)\log{(1-p)}$.
\end{proof}

\subsection*{Proof of Theorem~\ref{existence of hashing function}}
\begin{proof}
	Let $\mathcal{U}$ denotes the uniform distribution on all the function that maps from $S_n$ to $ T_n^{(r_n)}$. $\tilde{h}_n\sim \mathcal{U}$ indicates that  given $\bm{x},\bm{y}$, $\tilde{h}_n(\bm{x})$ and $\tilde{h}_n(\bm{y})$ are independently and uniformly distributed on $T_n^{(r_n)}$. We have 
	\begin{align*}
	&\mathop{Pr}_{\tilde{h}_n\sim \mathcal{U}}\left\{\exists \bm{x}\neq \bm{y}\in S_n, d_H(\tilde{h}_n(\bm{x}),\tilde{h}_n(\bm{y})) \leq B_n\right\}\\
	=& \mathop{Pr}_{\tilde{h}_n\sim \mathcal{U}}\left\{\bigcup_{\bm{x}\neq \bm{y}
		\in S} \{d_H(\tilde{h}_n(\bm{x}),\tilde{h}_n(\bm{y})) \leq B_n\}\right\}\\
	\leq& \sum_{\bm{x}\neq \bm{y}\in S_n}{\mathop{Pr}_{\tilde{h}_n\sim \mathcal{U}}\left\{d_H(\tilde{h}_n(\bm{x}),\tilde{h}_n(\bm{y})) \leq B_n \right\}}\\
	<& \frac{|S_n|^2}{2} \mathop{Pr}_{\tilde{h}_n\sim \mathcal{U}}\left\{d_H(\tilde{h}_n(\bm{x}),\tilde{h}_n(\bm{y})) \leq B_n \right\}\\
	< & \frac{|S_n|^2}{2} \frac{\sum_{k=0}^{\frac{B_n}{2}} \binom{r_n}{k}\cdot \binom{n-r_n}{k}} {\binom{n}{r_n}}
	\end{align*}
	From equation (\ref{monotonous}) in the proof of Theorem~\ref{upperbound for capacity}, we know that 
	\begin{equation}\sum_{k=0}^{\frac{B_n}{2}} \frac{\binom{r_n}{k}\cdot \binom{n-r_n}{k}}{ \binom{n}{r_n}} \leq 2 \frac{\binom{r_n}{\frac{B_n}{2}} \cdot \binom{n-r_n}{\frac{B_n}{2}}} {\binom{n}{r_n}}\end{equation}
	By using the condition in the statement
	\begin{equation}\log{\vert S_n \vert }\leq \frac{1}{2} \left(H(\alpha)-\alpha H(\beta)-(1-\alpha)H \left(\frac{\alpha\beta}{1-\alpha}\right)+o(1)\right)n,\end{equation}
	we derive 
	\begin{equation}\begin{split}
	&\mathop{Pr}_{\tilde{h}_n\sim \mathcal{U}}\left\{\exists \bm{x}\neq \bm{y}\in S_n, d_H(\tilde{h}_n(\bm{x}),\tilde{h}_n(\bm{y})) \leq B_n\right\} <1\\
	 \Rightarrow& \exists \hat{h}_n: S_n \rightarrow  T_n^{(r_n)}, \forall \bm{x},\bm{y}\in S_n, d_H(\hat{h}_n(\bm{x}),\hat{h}_n(\bm{y}))> B_n
	\end{split}
	\end{equation}
	Notice that $\hat{h}_n$ has Lipschitz constant $\frac{2r_n}{A_n}$ on $S_n$. To extend $\hat{h}_n$ to $h_n$ defined on $\{0,1\}^n$, we use Kirszbraun theorem \cite{kirszbraun}. 
	
	Consider the Euclidean space $R^n$ and its subspace $W^{n-1}=\{\bm{x} \in R^{n-1}| \sum_{i=1}^{n} x_i =0\}$. The relationship between Euclidean distance and Hamming distance is that $d(\bm{x},\bm{y})= \sqrt{d_H(\bm{x},\bm{y})}, \forall \bm{x},\bm{y} \in S_n$. Let $\omega$ denote the translation $\omega(\bm{x})=\bm{x}-\frac{r_n}{n} \bm{1}_n$, then we have
	$\omega\circ \hat{h}_n: S_n \rightarrow W_{n-1}$ is a  Lipschitz-continuous map from the subset $S_n$ of Hilbert space $(R^n,d)$ to Hilbert space $(W_{n-1},d)$ with Lipschitz constant $\sqrt{\frac{2r_n}{A_n}}$. According to Kirszbraun theorem, there exists a Lipschitz-continuous map $h_0$ that extends $\omega\circ \hat{h}_n$ and has the same Lipschitz constant:
	$h_0: R^n \rightarrow W^{n-1}$
	\begin{equation}
	\forall \bm{x},\bm{y} \in R^n, d(h_0(\bm{x}),h_0(\bm{y})) \leq \sqrt{\frac{2r_n}{A_n}} d(\bm{x},\bm{y})
	\end{equation}
	We define $h_n$ on $\{0,1\}^n$ as follows:
	\begin{equation}
	\forall \bm{x} \in \{0,1\}^n, h_n(\bm{x})_i = \left\{ \begin{array}{ll}
	0 & |h_0(\bm{x})_i+\frac{r_n}{n}| \leq |h_0(\bm{x})_i-\frac{n-r_n}{n}|\\
	1 & |h_0(\bm{x})_i+\frac{r_n}{n}| > |h_0(\bm{x})_i-\frac{n-r_n}{n}|
	\end{array} \right. 1 \leq i \leq n
	\end{equation}
	Given an arbitrary $\bm{y} \in S_n$, we have $\forall \bm{x} \in \{0,1\}^n$,
	\begin{equation}|h_0(\bm{x})_i-h_0(\bm{y})_i| < \frac{1}{2} \Rightarrow h_n(\bm{x})_i=h_n(\bm{y})_i\end{equation}
	Let $k$ denote the number of bits such that $|h_0(\bm{x})_i-h_0(\bm{y})_i| \geq  \frac{1}{2}$, we have 
	\begin{equation}\sqrt{\frac{k}{4}} \leq d(h_0(\bm{x}),h_0(\bm{y})) \leq \sqrt{\frac{2r_n}{A_n}} d(\bm{x},\bm{y}) \Rightarrow k \leq \frac{8r_n}{A_n} d_H(\bm{x},\bm{y})\end{equation}
	Therefore
	\begin{equation} d_H(h_n(\bm{x}),h_n(\bm{y})) =\sum_{i=1}^n |h_n(\bm{x})_i-h_n(\bm{y})_i| \leq k \leq \frac{8r_n}{A_n} d_H(\bm{x},\bm{y})\end{equation}
\end{proof}

\subsection*{Proof of Theorem~\ref{neural SMA}}
Theorem~\ref{neural SMA} guarantees stable memory allocation properties of a three-layer neural circuit with divisive inhibition. Key issue in the proof is to estimate high probability upper and lower bounds for output-layer proportion of activation, as well as expansion rates. From a high level point of view, the second layer of this network makes it possible for a wide range of input density to have approximately the same output density, while in the third layer, the parameter for divisive inhibition enables us to control the output density at an arbitrary level. The continuity and orthogonality condition is satisfied since random projection preserves locality.\\
Lemma~\ref{middle layer stablity} describes the probability of activation in the second layer:
\begin{proof}[Proof of Lemma~\ref{middle layer stablity}]
	By symmetry, the probability of $\beta^T\bm{x}>0$ and $\beta^T\bm{x}<0$ is equal. So we have $\mathop{Pr}\{\beta^T\bm{x}>0\}=\frac{1}{2}\left(1-\mathop{Pr}\{\beta^T\bm{x}>0\}\right).$
	We can then write this probability as:
	\begin{equation}Pr\{\beta^T\bm{x}=0\}=\sum_{i=1}^{\lfloor|x|/2\rfloor}\binom{|x|}{2i}\binom{2i}{i}p^{2i}(1-2p)^{|x|-2i}\end{equation}
	By applying Chernoff bound in relative entropy form, we have:
	\begin{equation}\sum_{i=2p|x|+1}^{\lfloor|x|/2\rfloor}\binom{|x|}{2i} 2^{2i}p^{2i}(1-2p)^{|x|-2i}\leq e^{-2(2\ln 2-1)p|x|}\end{equation}
	For the first $2p|x|$ terms, we use Stirling approximation $2^{-k}\binom{2k}{k}=\frac{1}{\sqrt{\pi k}}+O(\frac{1}{k})$ for each term. Since first $p|x|$ terms are smaller than terms with $k\in [p|x|+1,2p|x|]$, we have
	\begin{equation}\sum_{i=1}^{2p|x|}\binom{|x|}{2i} \binom{2i}{i}p^{2i}(1-2p)^{|x|-2i}\leq \frac{2}{\sqrt{\pi p|x|}}\sum_{i=p|x|+1}^{2p|x|}\binom{|x|}{2i} (2p)^{2i}(1-2p)^{|x|-2i}\leq \frac{1}{\sqrt{\pi p|x|}}\end{equation}
	which concludes the proof of this lemma.
\end{proof}
\begin{proof}[Proof of Lemma~\ref{middle one-bit change probablity}]
	Let $q$ be the index where $x_q\neq y_q$, we assume $x_q=0,y_q=1$ without loss of generality. On such circumstance, the only chance to make $g(x)\neq g(y)$ is $\langle \beta_q=1,\beta^Tx=0\rangle$ or $\langle\beta_q=-1,\beta^Tx=1\rangle$. Using the estimate in Lemma~\ref{middle layer stablity} we can obtain the probability of such events, which leads directly to the result.
\end{proof}
\begin{proof}[Proof of Lemma~\ref{CLT second-layer change probability} and Lemma~\ref{CLT third-layer change probability}]
	In the proof we use Gaussian approximation for binomial variables, the error term can be estimated via Berry-Esseen bound.\\
	Let the common bits of $x$ and $y$ be vector $w=x\cap y$, we have $x=w+\tilde{x},y=w+\tilde{y}$. We obtain independent random variables $\beta^Tw,\beta^T\tilde{x},\beta^T\tilde{y}$. By applying bivariate central limit theorem to them, we have
	\begin{equation}
	\left[
	\begin{matrix}
	\beta^Tx\\
	\beta^Ty
	\end{matrix}
	\right]\sim
	\mathcal{N}\left(
	p(1-2c)
	\left[
	\begin{matrix}
	|x|\\
	|y|
	\end{matrix}
	\right],
	v(c,p)
	\left[
	\begin{matrix}
	|x|&|w|\\
	|w|&|y|
	\end{matrix}
	\right]
	\right)
	\end{equation}
	Then the probability can be written as quadrant integral of bivariate Gaussian distribution, as appears in the lemmas.
	\begin{equation}
	Pr\left\{g(x)\neq g(y)\right\}\approx \int\int_{u\cdot v<0} 	\mathcal{N}\left(
	p(1-2c)
	\left[
	\begin{matrix}
	|x|\\
	|y|
	\end{matrix}
	\right],
	v(c,p)
	\left[
	\begin{matrix}
	|x|&|w|\\
	|w|&|y|
	\end{matrix}
	\right]
	\right)dudv
	\end{equation}
	The integral in Lemma~\ref{CLT second-layer change probability} centers at origin and admits a closed-form solution. While in Lemma~\ref{CLT third-layer change probability} we write it as univariate integral with CDF and PDF of Gaussian. Later we will estimate its value given specific parameters.\par
	According to Berry-Esseen inequality, the error introduced in cumulative density functions by Gaussian approximator is bounded with $O(\frac{E[|\beta_i|^3]}{var(\beta_i)^{1.5}\sqrt{n}})=O(\frac{1}{\sqrt{pn}})$.
\end{proof}
Now we proceed to prove Theorem~\ref{neural SMA}.
\begin{proof}[Proof of Theorem~\ref{neural SMA}]
	We give proofs for three requirements one by one:\par
	\textbf{Stability}: Let $s_1$ denote the probability for middle-layer neurons to fire. According to Lemma~\ref{middle layer stablity}, $s_1$ is restricted in a narrow range: $0\leq \frac{1}{2}-s_1\leq \frac{1}{2\sqrt{\pi p |x|}}+\frac{1}{2}e^{-2(2\ln 2-1)p|x|}$.\\
	Now that each unit in third layer depends upon the sum of several $i.i.d.$ random variables. By applying central limit theorem to the sum, we can estimate the probability for a output-layer neuron to fire is approximately \begin{equation}Pr\{h(x)_i=1\}\approx\Phi\left(\sqrt{n}\frac{(1-2c)ps_1}{\sqrt{v(c,ps_1)}}\right)\end{equation}
	By plugging in the parameter  $C_2=\frac{1}{2}-
	\frac{s}{2\sqrt{2-s^2}}$ and $s=\Phi^{-1}\left(\frac{r_n}{n}\right)\sqrt{\frac{2}{np}}$, we can show that 
	\begin{equation}\leq\frac{r_n}{n}-Pr\{h(x)_i=1\}\leq \frac{log(n/r_n)}{\sqrt{s_0n^{1-\gamma}}}\end{equation}
	Since all output units are $i.i.d.$ with respect to randomness of last layer, we can therefore apply Chernoff's bound to conclude the stability argument.\par
	\textbf{Continuity}:
	The technique used in our proof is based on bounding the output difference between two vector who has exact one different bit. Fix $d_H(x,y)=L$,  we construct a sequence $x=v_0,v_1,v_2,\cdots,v_L=y$ with $d_H(v_{i-1},v_i)=1$. \\
	For $v_{i-1}$ and $v_{i}$, consider the probability for each second-layer neuron to be different. According to Lemma~\ref{middle one-bit change probablity}, we can estimate $Pr\{g(v_i)\neq g(v_{i-1})\}\approx \frac{1}{\pi}\cos^{-1}\frac{|v_{i-1}\cap v_{i}|}{\sqrt{|v_{i}|\cdot |v_{i-1}|}}$. By union bound we have the following probabilistic triangle inequality:
	\begin{equation}Pr\{g(x)\neq g(y)\}\leq Pr\{\bigcup_{i=1}^{L}g(v_i)\neq g(v_{i-1})\}\leq \sum_{i=1}^{L}Pr\{g(v_i)\neq g(v_i-1)\}
	\end{equation}
	Lemma~\ref{middle layer stablity} tells us that both $x$ and $y$ activate approximately $\frac{n}{2}$ neurons in the middle layer, which satisfies the requirement of Lemma~\ref{CLT third-layer change probability}. By applying
	Lemma~\ref{CLT third-layer change probability} with $\eta=O(L\sqrt{\frac{p}{s_0n}})$, we derive the following (approximate) upper bound for probability of changed bits in third layer:
	\begin{equation}O(1)\cdot\int_{\mathbb{R}}{\varphi(t)\Phi\left(-\frac{(s_0n^{1+\gamma})^{\frac{1}{4}}}{\sqrt{L}}|t-\Phi^{-1}(\frac{r_n}{n})|\right)dt}\end{equation}
	We estimate the integral by dividing it into three parts:\\
	Let $t_{1},t_{2}=\Phi^{-1}(\frac{r_n}{n})\pm\frac{2\log n\sqrt{L}}{(s_0n^{1+\gamma})^{\frac{1}{4}}}$ respectively. We have
	\begin{equation}\int_{(t_1,t_2)}{\varphi(t)\Phi\left(-\frac{(s_0n^{1+\gamma})^{\frac{1}{4}}}{\sqrt{L}}|t-\Phi^{-1}(\frac{r_n}{n})|\right)dt}\leq \int_{(t_1,t_2)}\phi(t)dt\leq \frac{1}{\sqrt{2\pi}}(t_2-t_1)\end{equation}
	For the integral value outside this interval, we bound it with Gaussian tail:
	\begin{equation}\int_{t>t_2}{\varphi(t)\Phi\left(-\frac{(s_0n^{1+\gamma})^{\frac{1}{4}}}{\sqrt{L}}|t-\Phi^{-1}(\frac{r_n}{n})|\right)dt}\leq \int_{t>t_2}{\varphi(t)\Phi\left(-2\log n\right)dt}=O\left( \frac{1}{n^2}\right)
	\end{equation}
	(For $t<t_1$, we have very similar bounds)\\
	Summarizing the calculation above, we can get an approximate upper bound \begin{equation}Pr\{h(x)_i\neq h(y)_i\}=O\left(\frac{\log n\sqrt{L}}{(s_0n^{1+\gamma})^{\frac{1}{4}}}\varphi(\Phi^{-1}(\frac{r_n}{n}))\right)\end{equation}
	So the probability for each output neuron to differ is upper bounded with $\frac{\mu_n\sqrt{L}}{n}$. We then apply the Chernoff bound in relative entropy form and conclude the probabilistic upper bound.\\
	\textbf{Orthogonality}:
	For $x,y\in \{0,1\}^n$ such that $d_H(x,y)\geq an$, we have 
	\begin{equation}\frac{|x\cap y|}{\sqrt{|x|\cdot |y|}}\leq \sqrt{\frac{|x\cap y|}{|x\cup y|}}\leq \sqrt{1-a}\end{equation}
	According to Lemma~\ref{CLT second-layer change probability}, the probability of making differences at second layer is at least \begin{equation}Pr\{g(x)\neq g(y)\}\geq b_1\triangleq\frac{1}{\pi}\cos^{-1}\sqrt{1-a}\end{equation}
	Then we apply Lemma~\ref{CLT third-layer change probability} with $\eta=b_1$:
	\begin{equation}Pr\{h(x)_i\neq h(y)_i\}\geq\int_{\mathbb{R}}\varphi(\frac{t}{\sqrt{1-\frac{1}{2}b_1}})\Phi(\frac{|t-\Phi^{-1}(r_n/n)|}{\sqrt{b_1}})dt\end{equation}
	Assuming $b_1$ constant, we can obtain the following lower bound for this integral:
	\begin{equation}
	\begin{split}
	Pr\{h(x)_i\neq h(y)_i\}\geq &\int_{\mathbb{R}}\varphi(\frac{t}{\sqrt{1-\frac{1}{2}b_1}})\Phi(\frac{|t-\Phi^{-1}(r_n/n)|}{\sqrt{b_1}})dt\\
	\geq&\int_{|t-\Phi^{-1}(\frac{r_n}{n})|\leq \sqrt{b_1}}\varphi(\frac{t}{\sqrt{1-\frac{1}{2}b_1}})\Phi(\frac{|t-\Phi^{-1}(r_n/n)|}{\sqrt{b_1}})dt\\
	\geq&\int_{|t-\Phi^{-1}(\frac{r_n}{n})|\leq \sqrt{b_1}}\varphi(\frac{t}{\sqrt{1-\frac{1}{2}b_1}})\Phi(-1)dt\\
	\geq&\Phi(-1)\Phi\left(\frac{\Phi^{-1}(r_n/n)+\sqrt{b_1}}{\sqrt{1-\frac{1}{2}b_1}}\right)-\Phi(-1)\Phi\left(\frac{\Phi^{-1}(r_n/n)-\sqrt{b_1}}{\sqrt{1-\frac{1}{2}b_1}}\right)\\
	\geq&\frac{br_n}{n}
	\end{split}
	\end{equation}
	The value of $b$ depends upon neither $r_n$ nor $n$, but only $b_1$, which is assumed to be a constant. So we have proved a lower bound for the probability of third-layer nodes to be different. Since those nodes are $i.i.d.$ with respect to the randomness in the last layer. By applying Chernoff bound we can get this result.
\end{proof}
\bibliographystyle{plain}
\bibliography{sma}

\begin{thebibliography}{10}

\bibitem{signcons}
Zeyuan Allen-Zhu, Rati Gelashvili, Silvio Micali, and Nir Shavit.
\newblock Sparse sign-consistent johnson--lindenstrauss matrices: Compression
  with neuroscience-based constraints.
\newblock {\em Proceedings of the National Academy of Sciences},
  111(47):16872--16876, 2014.

\bibitem{andersen2006hippocampus}
Per Andersen, Richard Morris, David Amaral, Tim Bliss, and John O'Keefe.
\newblock {\em The hippocampus book}.
\newblock Oxford University Press, USA, 2006.

\bibitem{DBLP:conf/focs/ArriagaV99}
Rosa~I. Arriaga and Santosh Vempala.
\newblock An algorithmic theory of learning: Robust concepts and random
  projection.
\newblock In {\em Proc. of FOCS}, pages 616--623, 1999.

\bibitem{chvatal1979tail}
Va{\v{s}}ek Chv{\'a}tal.
\newblock The tail of the hypergeometric distribution.
\newblock {\em Discrete Mathematics}, 25(3):285--287, 1979.

\bibitem{debanne1999heterogeneity}
Dominique Debanne, Beat~H G{\"a}hwiler, and Scott~M Thompson.
\newblock Heterogeneity of synaptic plasticity at unitary ca3--ca1 and ca3--ca3
  connections in rat hippocampal slice cultures.
\newblock {\em The Journal of Neuroscience}, 19(24):10664--10671, 1999.

\bibitem{doiron2001subtractive}
Brent Doiron, Andr{\'e} Longtin, Neil Berman, and Leonard Maler.
\newblock Subtractive and divisive inhibition: effect of voltage-dependent
  inhibitory conductances and noise.
\newblock {\em Neural Computation}, 13(1):227--248, 2001.

\bibitem{DBLP:journals/neco/FeldmanV09}
Vitaly Feldman and Leslie~G. Valiant.
\newblock Experience-induced neural circuits that achieve high capacity.
\newblock {\em Neural Computation}, 21(10):2715--2754, 2009.

\bibitem{hopfield}
John~J Hopfield.
\newblock Neural networks and physical systems with emergent collective
  computational abilities.
\newblock {\em Proceedings of the national academy of sciences},
  79(8):2554--2558, 1982.

\bibitem{lsh}
Piotr Indyk and Rajeev Motwani.
\newblock Approximate nearest neighbors: towards removing the curse of
  dimensionality.
\newblock In {\em Proc. of STOC}, pages 604--613. ACM, 1998.

\bibitem{jlopt}
T.~S. Jayram and David~P. Woodruff.
\newblock Optimal bounds for johnson-lindenstrauss transforms and streaming
  problems with subconstant error.
\newblock {\em {ACM} Transactions on Algorithms}, 9(3):26, 2013.

\bibitem{kirszbraun}
M~Kirszbraun.
\newblock {\"U}ber die zusammenziehende und lipschitzsche transformationen.
\newblock {\em Fundamenta Mathematicae}, 1(22):77--108, 1934.

\bibitem{mejias2014subtractive}
Jorge~F Mejias, Alexandre Payeur, Erik Selin, Leonard Maler, and Andr{\'e}
  Longtin.
\newblock Subtractive, divisive and non-monotonic gain control in feedforward
  nets linearized by noise and delays.
\newblock {\em Frontiers in computational neuroscience}, 8, 2014.

\bibitem{DBLP:conf/podc/PapadimitriouV15}
Christos~H. Papadimitriou and Santosh~S. Vempala.
\newblock Cortical computation.
\newblock In {\em Proceedings of the 2015 {ACM} Symposium on Principles of
  Distributed Computing, {PODC} 2015, Donostia-San Sebasti{\'{a}}n, Spain, July
  21 - 23, 2015}, pages 1--2, 2015.

\bibitem{papadimitriou2015cortical}
Christos~H Papadimitriou and Santosh~S Vempala.
\newblock Cortical learning via prediction.
\newblock In {\em Proc. of COLT}, 2015.

\bibitem{neurobook}
Leslie~G. Valiant.
\newblock {\em Circuits of the mind}.
\newblock Oxford University Press, 1995.

\bibitem{valiant2000neuroidal}
Leslie~G Valiant.
\newblock A neuroidal architecture for cognitive computation.
\newblock {\em Journal of the ACM (JACM)}, 47(5):854--882, 2000.

\bibitem{memo}
Leslie~G. Valiant.
\newblock Memorization and association on a realistic neural model.
\newblock {\em Electronic Colloquium on Computational Complexity {(ECCC)}},
  (004), 2005.

\bibitem{valiant2006quantitative}
Leslie~G Valiant.
\newblock A quantitative theory of neural computation.
\newblock {\em Biological cybernetics}, 95(3):205--211, 2006.

\bibitem{hippo}
Leslie~G. Valiant.
\newblock The hippocampus as a stable memory allocator for cortex.
\newblock {\em Neural Computation}, 24(11):2873--2899, 2012.

\bibitem{valiant2014must}
Leslie~G Valiant.
\newblock What must a global theory of cortex explain?
\newblock {\em Current opinion in neurobiology}, 25:15--19, 2014.

\bibitem{DBLP:conf/icml/YiCP15}
Xinyang Yi, Constantine Caramanis, and Eric Price.
\newblock Binary embedding: Fundamental limits and fast algorithm.
\newblock In {\em Proc. of ICML}, pages 2162--2170, 2015.

\end{thebibliography}
\end{document}